\documentclass[runningheads,a4paper,orivec]{llncs}

\usepackage{algorithm}
\usepackage[noend]{algpseudocode}
\usepackage[utf8]{inputenc}
\usepackage[T1]{fontenc}
\usepackage{mathtools}
\usepackage{dsfont}
\usepackage{color}
\usepackage[mathscr]{euscript}
\usepackage{comment}
\usepackage{graphicx}
\usepackage{placeins}
\usepackage{float}
\usepackage{multirow}
\usepackage{pdflscape}
\usepackage{caption}
\usepackage{pgfplots}
\usepackage{subcaption}
\usepackage{mathabx}
\usepackage{tikz}
\usepackage{hyperref}
\usepackage[commandnameprefix=ifneeded,final]{changes}
\usepackage{orcidlink}
\usepackage[T1]{fontenc}

\usepackage[normalem]{ulem}

\usepackage{newtxtext}       %

\usepackage[varvw]{newtxmath}       %

\newcommand{\mcF}{\mathcal{F}}

\newcommand{\mcN}{\mathcal{N}}

\newcommand{\mcR}{\mathcal{R}}

\newcommand{\mbR}{\mathbb{R}}

\newcommand{\mbZ}{\mathbb{Z}}

\newcommand{\sign}{\text{sgn}}
\newcommand{\tst}{\tau^\star}

\usepackage{graphicx}

\makeatletter
\g@addto@macro\normalsize{%
	\setlength\abovedisplayskip{5pt plus 2pt minus 2pt}%
	\setlength\belowdisplayskip{5pt plus 2pt minus 2pt}%
	\setlength\abovedisplayshortskip{0pt plus 2pt}%
	\setlength\belowdisplayshortskip{3pt plus 2pt minus 1pt}}
\makeatother
\begin{document}
    \title{Neural Disjunctive Normal Forms for Learning Exact Rank Filters}
    
	\author{Diego Marcondes\inst{1}\orcidlink{0000-0002-6087-4821} \and
		Wonder A. L. Alves\inst{2}\orcidlink{0000-0003-0430-950X} \and Ronaldo~F.~Hashimoto\inst{3}\orcidlink{0000-0002-6399-8790}}
	\authorrunning{D. Marcondes et al.}
	\institute{Mathematical Sciences Institute and France-Australia Mathematical Sciences and Interactions ANU-CNRS International Research Lab, The Australian National University \and
		PPGI, Universidade Nove de Julho, S\~{a}o Paulo, Brazil\and IME, Universidade de S\~{a}o Paulo, S\~{a}o Paulo, Brazil}
	\maketitle              %
	\begin{abstract}
		Gradient-based learning of morphological operators is hindered by locally sparse derivatives with respect to the structuring elements. We develop a method for learning rank filters that retains exact operator evaluation while distributing the derivatives across the window. Using threshold decomposition, we construct a neural disjunctive normal form representation and establish a sufficient condition for exactness with integer-valued images and structuring elements. We then derive a smooth surrogate from the unique root of a smooth threshold equation. Implicit differentiation gives closed-form derivatives with normalised local weights, whose distribution is controlled by a smoothing parameter. These derivatives are used in a straight-through procedure: the forward pass evaluates the exact rank filter, while the backward pass uses the surrogate. For fixed ranks, the procedure learns real-valued structuring elements of any rank filter, including the median, and extends to compositions, infima, and suprema of rank filters. Experiments on image denoising using MNIST and Fashion-MNIST and retinal vessel segmentation show that, starting from a generic random initialisation, the proposed approach consistently improves performance over training with exact subgradients in every setting. In image denoising, our best network also outperforms smooth morphological networks in all settings. On vessel segmentation, the learned erosions recover structuring elements that fit the vessels, as theoretically expected, whereas most of those trained with exact subgradients remain close to their initialisation.

	\end{abstract}

    \section{Introduction}
\label{SecIntro}

Rank filters, including erosions and dilations, describe image transformations through local order statistics and structuring elements \cite{maragos1987morphological}. Learning the structuring elements from paired images adapts these transformations to a task, %
and the resulting operator is determined by its window, rank, and structuring element. %
Preserving exact evaluation keeps this description valid throughout training and inference. However, at each pixel where the selected order statistic is attained at a unique position, %
the derivative with respect to the structuring element is nonzero only at that position. This local gradient sparsity can hinder optimisation and make training sensitive to initialisation \cite{blusseau2024training,dimitrova2025learning,dimitrova2026optimizing}.

Morphological networks have been trained with the exact operators and their subgradient \cite{franchi2020deep}, which inherits this sparsity. Recent works keep this setting and address it through guidelines on initialisation and learning rates \cite{blusseau2024training}, and, for a single supremum of erosions, through data-driven initialisations and denser descent updates \cite{dimitrova2025learning,dimitrova2026optimizing}. Alternatively, smooth layers can replace the operators by differentiable approximations \cite{bottenmuller2025improving,hermary2022learning}, so that the trained operator only approximates a morphological one. As a compromise, older alternatives keep the exact operator in the forward pass and modify only its derivative. Salembier \cite{salembier1992adaptive} adapted rank filters through an implicit formulation of the rank by sign functions, and Pessoa and Maragos \cite{pessoa1998mrl} replaced the rank indicator vector by smoothed impulses that spread the gradient over the window. This smoothing is a heuristic design choice rather than the derivative of a function representing the filter.
 
 In this paper, we rely on the threshold decomposition \cite{maragos1987morphological,wendt1986stack} of rank filters to represent them as a sum, over thresholds, of order statistics of Boolean variables, and adapt the semi-symbolic layer of pix2rule \cite{cingillioglu2021pix2rule} to obtain a neural disjunctive normal form representation of these order statistics, which yields a smooth representation of rank filters that is exact on the integer grid when the smoothing parameter is large enough. Implicit differentiation of the root of a smooth threshold equation yields a surrogate derivative with normalised weights over the window which have the form of the smoothed impulses of Pessoa and Maragos \cite{pessoa1998mrl}. However, unlike their heuristic approach, our weights are obtained as the exact derivatives of a mathematically defined smooth surrogate of the filter, addressing gradient sparsity in a principled way.
 
 For fixed ranks, we use these derivatives in a straight-through procedure \cite{bengio2013estimating} that learns real-valued structuring elements: the forward pass evaluates the exact filter, while the backward pass uses the surrogate. The smoothing parameter controls the distribution of the local derivative across the window, while the evaluated filter remains exact and independent of that choice. The procedure applies to any rank, such as the median, which smooth layers do not usually cover, and to compositions, infima and suprema of rank filters. It thus, in principle, extends to grey-level images the framework of discrete morphological neural networks \cite{marcondes2024discrete}, which learn exact binary operators combined by composition, infimum and supremum in architectures designed from prior knowledge.

We evaluate the procedure on image denoising and retinal vessel segmentation, comparing it with direct differentiation (subgradient) of exact operators and, for denoising, with smooth morphological networks. These comparisons assess whether redistributing the derivatives improves learning while retaining exact operator evaluation. Starting from a generic random initialisation, not tailored to the data, our method consistently improves on the exact subgradient in every setting, without requiring the specialised initialisation strategies employed by recent exact-training approaches \cite{blusseau2024training,dimitrova2025learning,dimitrova2026optimizing} and our best network outperforms the smooth networks in all denoising settings, while the trained operator remains exact. Moreover, just as the layers of discrete morphological neural networks specialise into the expected steps of a task, such as denoising followed by boundary extraction, without being constrained to do so \cite{marcondes2025unrestricted}, the erosions learned by our method for retinal vessels recover structuring elements of different sizes and orientations that fit into the vessels, as theoretically expected \cite{zana2001segmentation}, whereas most of those learned with the exact subgradient remain close to their random initialisation.

Our main contributions are therefore: (i) a theoretically grounded approximation of rank-filter derivatives, rather than a heuristic redistribution of gradients \cite{pessoa1998mrl}; (ii) exact forward evaluation combined with smoother backward gradients, mitigating the dependence on specialised initialisation strategies of recent exact-training methods \cite{blusseau2024training,dimitrova2025learning,dimitrova2026optimizing}; and (iii) the demonstration of effective training on non-trivial image processing tasks, starting from generic random initialisations, improving upon exact-subgradient training and smooth morphological layers.

    \section{Theoretical background}
\label{SecBackground}

	Let $E = \mbZ^{2}$. Fix $m \in \mbZ_{+}$, define $M \coloneqq \{0,1,\dots,m\}$ and let $\mcF \coloneqq M^{E}$ be the space of all functions from $E$ to $M$. Elements $f \in \mcF$ are greyscale images in which the value $f(x)$ is the intensity of the pixel $x \in E$. Let $W \subset E, |W| \geq 3,$ be a finite set with an odd number of elements and for $x \in E$ let $W + x \coloneqq \{y + x: y \in W\}$ be the translation of $W$ by $x$. We denote by $f|_{A}:A \to M$ the restriction of $f$ to a set $A \subset E$. We identify the restriction of $f$ to $W+x$ with the function on $W$ given by $y \mapsto f(x+y)$, and keep the notation $f|_{W+x}$ for this translated function. Thus, $f|_{W+x}$ is an element of $\mcF_{W} \coloneqq M^{W}$. We denote general elements of $\mcF_{W}$ by $f_{W}$.

	Consider in $\mcF_{W}$ the point-wise partial order $f_{W} \leq g_{W}$ iff $f_{W}(y) \leq g_{W}(y)$ for all $y \in W$. The poset $(\mcF_{W},\leq)$ is actually a complete lattice, with infimum and supremum denoted by $\wedge$ and $\vee$, respectively. We may also use $\wedge$ and $\vee$ to represent the minimum and maximum of real numbers, and it should be clear from context what these symbols mean. We define the sign function by $\sign(z)=-1$ for $z<0$ and $\sign(z)=+1$ for $z\geq 0$.

\subsection{Rank filters and threshold decomposition}

   We recall rank filters and their threshold decomposition, and refer to \cite{maragos1987morphological,wendt1986stack} for proofs. For a function $g: W \to \mbZ$, let $[g]_{(1)} \leq [g]_{(2)} \leq \cdots \leq [g]_{(|W|)}$ be its order statistics, that is an ordering of the vector $(g(y): y \in W)$. The $q$-th rank filter with structuring element $a: W \to \mbZ$ is given by
	\begin{align}
		\label{gen_filter}
		\psi_{a,q}(f)(x) = ([f|_{W + x} - a]_{(q)} \vee 0) \wedge m
	\end{align}
	for all $f \in \mcF$ and $x \in E$, and for $q = 1,\dots,|W|$. As special cases, we have the erosion by $a$ when $q = 1$,
	a dilation by the structuring element $b:-W\to\mbZ$ when $q=|W|$, where $-W=\{-y:y\in W\}$ and $b(-y)=-a(y)$ for all $y\in W$,
	and the usual median filter when $q = (|W| + 1)/2$ and $a(y)=0$ for all $y \in W$. The following proposition states the threshold decomposition for integer-valued $f_W$ and $a$, with thresholds applied to $f_W-a$. Figure~\ref{fig_threshold_decomposition} illustrates Proposition~\ref{main_prop}.
	
	\begin{proposition}
		\label{main_prop}
		For $a: W \to \mbZ$ and $q = 1,\dots,|W|$, the $q$-th rank filter satisfies
		\begin{align}
			\label{rank_sum}
			\psi_{a,q}(f)(x) = \sum_{t = 1}^{m} \left[\left(\mathds{1}\{f|_{W + x}(y) - a(y) \geq t\}: y \in W\right)\right]_{(q)}
		\end{align}
		for all $f \in \mcF$ and $x \in E$. In particular, since the terms in the sum above are non-increasing in $t$, $\psi_{a,q}(f)(x) = 0$ if $\left[\left(\mathds{1}\{f|_{W + x}(y) - a(y) \geq 0\}: y \in W\right)\right]_{(q)} = 0$ and
        \begin{align}
            \label{max_form_rank}
            \psi_{a,q}(f)(x) = \max\{t \in M: \left[\left(\mathds{1}\{f|_{W + x}(y) - a(y) \geq t\}: y \in W\right)\right]_{(q)} = 1\},
        \end{align}
        otherwise.
	\end{proposition}
	
	Hence, a rank filter is determined by the $q$-th order statistics of the Boolean variables $\mathds{1}\{f|_{W + x}(y) - a(y) \geq t\}$, $y \in W$, for $t = 1,\dots,m$, and \eqref{max_form_rank} avoids summing over the grey levels. In particular, since the terms of the sum in \eqref{rank_sum} are non-increasing in $t$, the maximum in \eqref{max_form_rank} can be computed by a bisection algorithm with complexity $\mathcal{O}(\log_{2}(m + 1))$ against $\mathcal{O}(m)$ for the sum in \eqref{rank_sum}. 
	
	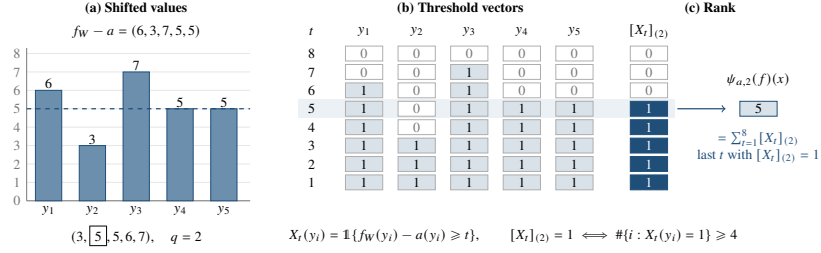
\begin{figure}[ht]
		\centering
		\resizebox{0.9\textwidth}{!}{\definecolor{rkdark}{RGB}{31,78,121}
\definecolor{rkbar}{RGB}{107,143,173}
\definecolor{rklight}{RGB}{217,226,233}
\definecolor{rkband}{RGB}{236,241,245}
\definecolor{rkgrey}{RGB}{120,120,120}
\begin{tikzpicture}[
    every node/.style = {inner sep = 1pt},
    cell/.style = {draw = rkgrey!60, minimum width = 0.85cm, minimum height = 0.34cm, inner sep = 0pt},
    one/.style = {cell, fill = rklight},
    zero/.style = {cell, fill = white, text = rkgrey},
    bone/.style = {cell, draw = rkdark, fill = rkdark, text = white}]
    \def\vals{6,3,7,5,5}
    \def\h{0.42}

    \node[font = \bfseries] at (3,4.4) {(a) Shifted values};
    \node at (3,3.86) {$f_{W} - a = (6,3,7,5,5)$};
    \foreach \t in {0,...,8} {
        \draw[black!10] (0.5,\h*\t) -- (5.6,\h*\t);
        \node[left, text = rkgrey, font = \small] at (0.4,\h*\t) {\t};
    }
    \draw[rkgrey] (0.5,0) -- (0.5,8.3*\h);
    \draw[rkgrey] (0.5,0) -- (5.6,0);
    \foreach \v [count = \i] in \vals {
        \draw[rkdark, fill = rkbar] (\i - 0.3,0) rectangle (\i + 0.3,\h*\v);
        \node[above] at (\i,\h*\v) {\v};
        \node[below = 2pt] at (\i,0) {$y_{\i}$};
    }
    \draw[rkdark, thick, dashed] (0.5,5*\h) -- (5.6,5*\h);
    \node at (3,-0.8) {$(3,\boxed{5},5,6,7), \quad q = 2$};

    \begin{scope}[shift = {(6.6,0)}]
        \node[font = \bfseries] at (3.8,4.4) {(b) Threshold vectors};
        \fill[rkband] (0.1,5*\h - 0.21) rectangle (8.7,5*\h + 0.21);
        \node at (0.4,3.86) {$t$};
        \foreach \t in {1,...,8} \node at (0.4,\h*\t) {\t};
        \foreach \v [count = \i] in \vals {
            \node at (1.2*\i + 0.4,3.86) {$y_{\i}$};
            \foreach \t in {1,...,8} {
                \pgfmathsetmacro{\on}{\v >= \t ? 1 : 0}
                \ifdim \on pt > 0.5pt
                    \node[one] at (1.2*\i + 0.4,\h*\t) {1};
                \else
                    \node[zero] at (1.2*\i + 0.4,\h*\t) {0};
                \fi
            }
        }

        \node[font = \bfseries] at (9.5,4.4) {(c) Rank};
        \node at (8.1,3.86) {$[X_{t}]_{(2)}$};
        \foreach \t in {1,...,8} {
            \ifnum \t < 6
                \node[bone] at (8.1,\h*\t) {1};
            \else
                \node[zero] at (8.1,\h*\t) {0};
            \fi
        }
        \draw[->, rkdark, thick] (8.75,5*\h) -- (9.85,5*\h);
        \node at (10.6,5*\h + 0.65) {$\psi_{a,2}(f)(x)$};
        \node[cell, draw = rkdark, fill = rklight] at (10.6,5*\h) {5};
        \node[text = rkdark, font = \small, align = center] at (10.6,5*\h - 0.85) {$= \sum_{t = 1}^{8} [X_{t}]_{(2)}$\\last $t$ with $[X_{t}]_{(2)} = 1$};

        \node at (5,-0.8) {$X_{t}(y_{i}) = \mathds{1}\{f_{W}(y_{i}) - a(y_{i}) \geq t\}, \qquad [X_{t}]_{(2)} = 1 \iff \#\{i : X_{t}(y_{i}) = 1\} \geq 4$};
    \end{scope}
\end{tikzpicture}}
		\caption{\footnotesize Threshold decomposition for $m=8$, $q=2$ and $f_W-a=(6,3,7,5,5)$, with $W=\{y_1,\ldots,y_5\}$. %
		}
		\label{fig_threshold_decomposition}
	\end{figure}
    
    \subsection{Semi-symbolic neural layers}
    
    A neural network layer, called \textit{semi-symbolic} layer, was proposed by \cite{cingillioglu2021pix2rule} to approximate conjunctions (inf) and disjunctions (sup) of symbolic rules, and in particular of Boolean variables, based on the disjunctive normal form (DNF). By adapting this layer, one can exactly represent any order statistic of Boolean variables, as follows.
    
    Let $X: W \to \{-1,+1\}$ be a Boolean function on $W$. The neural DNF layer $\eta(X) \in \{-1,+1\}$ applied to $X$ with weights $v \in \mbR^W$ is given by
    \begin{align}
    	\label{node}
    	\eta(X) = \sign\left(\sum_{y \in W} v_{y} \, X(y) + \beta(\delta)\right) & & \text{ with } \beta(\delta) = \delta(\max_{y \in W} |v_y| - \sum_{y \in W} |v_{y}|)
    \end{align}
    for $\delta \in [-1,+1]$. The layer of \cite{cingillioglu2021pix2rule} uses $\tanh$ in place of $\sign$, which yields a smooth approximation in $[-1,+1]$, but we consider $\sign$ to obtain an exact representation.
    
    Let $X_{(1)} \leq X_{(2)} \leq \cdots \leq X_{(|W|)}$ be an ordering of the values of $X$. In particular, $X_{(1)} = \min_{y} X(y)$ and $X_{(|W|)} = \max_y X(y)$. We show that, by properly choosing\footnote{In the special case of $\delta = \pm 1$, the respective order statistic, the minimum and the maximum, can be represented by non-constant parameters $v \in \mbR^{W}$ with $v_{y} > \max_z v_z/2$ for all $y \in W$. We refer to \cite{cingillioglu2021pix2rule} for more details.} $v$ and $\delta$, any order statistic $X_{(q)}$ of $X$ can be represented by $\eta(X)$. A proof of Lemma \ref{lem1} is presented in Section \ref{SecProof}. An example of a DNF representation is presented in Figure \ref{fig_dnf_example}.
    
    \begin{lemma}
    	\label{lem1}
    	Fix $q \in \{1,\dots,|W|\}$. Let $v_{y} = \gamma > 0$ for all $y \in W$ and 
    	\begin{equation}
    		\label{delta_formula}
    		\delta = \frac{-2q + |W| + 1}{|W| - 1}.
    	\end{equation}
    	Then, $\eta(X) = X_{(q)}$ for all $X \in \{-1,+1\}^{W}$. In particular, $\eta(X) = \max_{y}\, X(y)$ if $\delta = -1$, $\eta(X) = \min_{y}\, X(y)$ if $\delta = +1$ and $\eta(X) = \text{Median}_{y}\, X(y)$ if $\delta = 0$.
    \end{lemma}

   \begin{figure}[!t]
    	\centering
    	\resizebox{0.9\textwidth}{!}{\definecolor{rkdark}{RGB}{31,78,121}
\definecolor{rklight}{RGB}{217,226,233}
\definecolor{rkband}{RGB}{236,241,245}
\definecolor{rkgrey}{RGB}{120,120,120}
\begin{tikzpicture}[
    every node/.style = {inner sep = 1pt},
    cell/.style = {draw = rkgrey!60, minimum width = 0.75cm, minimum height = 0.36cm, inner sep = 0pt},
    plus/.style = {cell, fill = rklight},
    minus/.style = {cell, fill = white, text = rkgrey},
    bpos/.style = {cell, draw = rkdark, fill = rkdark, text = white},
    op/.style = {draw = rkdark, minimum size = 0.55cm, inner sep = 0pt},
    arr/.style = {->, rkdark}]

    \node[font = \bfseries] at (2.9,3.1) {(a) Neural DNF layer};
    \foreach \x [count = \i] in {+1,-1,+1,+1,+1} {
        \pgfmathsetmacro{\yy}{2.5 - 0.5*\i}
        \ifnum \x > 0
            \node[plus] (x\i) at (1.1,\yy) {$+1$};
        \else
            \node[minus] (x\i) at (1.1,\yy) {$-1$};
        \fi
        \node[left] at (0.7,\yy) {$X(y_{\i})$};
        \draw[arr] (x\i.east) -- (2.6,1);
    }
    \node[font = \small, text = rkdark] at (1.95,2.35) {$v_{y} = 1$};
    \node[op, circle, fill = white] (s) at (2.9,1) {$\Sigma$};
    \node (b) at (3.45,2.35) {$\beta(\delta) = -2$};
    \draw[arr] (b) -- (s);
    \node[font = \small] at (2.9,0.45) {$3 - 2 = 1$};
    \node[op, rectangle, minimum width = 0.8cm] (g) at (4.15,1) {$\mathrm{sgn}$};
    \node[bpos, minimum width = 0.6cm] (o) at (5.3,1) {$+1$};
    \node at (5.3,1.55) {$\eta(X)$};
    \draw[arr] (s) -- (g);
    \draw[arr] (g) -- (o);
    \node at (2.9,-0.65) {$\delta = \frac{-2q + |W| + 1}{|W| - 1} = \frac{1}{2}, \quad \beta(\delta) = -4\delta = -2$};

    \begin{scope}[shift = {(6.9,0)}]
        \node[font = \bfseries] at (4.3,3.1) {(b) Output and rank};
        \fill[rkband] (5.94,-0.05) rectangle (6.86,2.62);
        \draw[rkdark, thick, dashed] (5.925,-0.1) -- (5.925,2.7);
        \foreach \lab [count = \r] in {{\#\{y : X(y) = +1\}},{\sum_{y} X(y)},{\sum_{y} X(y) - 2},{\eta(X)},{X_{(2)}}} {
            \node[left] at (2.05,2.95 - 0.55*\r) {$\lab$};
        }
        \foreach \k in {0,...,5} {
            \pgfmathtruncatemacro{\s}{2*\k - 5}
            \pgfmathtruncatemacro{\sb}{2*\k - 7}
            \pgfmathsetmacro{\x}{2.6 + 0.95*\k}
            \node at (\x,2.4) {\k};
            \node at (\x,1.85) {$\s$};
            \node at (\x,1.3) {$\sb$};
            \ifnum \k > 3
                \node[bpos] at (\x,0.75) {$+1$};
                \node[bpos] at (\x,0.2) {$+1$};
            \else
                \node[minus] at (\x,0.75) {$-1$};
                \node[minus] at (\x,0.2) {$-1$};
            \fi
        }
        \node at (4.3,-0.65) {$\eta(X) = +1 \iff \#\{y : X(y) = +1\} \geq 4 \iff X_{(2)} = +1$};
    \end{scope}
\end{tikzpicture}}
    	\caption{\footnotesize The neural DNF layer with $v_{y} = 1$ and $\delta = 1/2$ computes the $q = 2$ order statistic of $X \in \{-1,+1\}^{W}$, $|W| = 5$.} %
    	\label{fig_dnf_example}
    \end{figure}
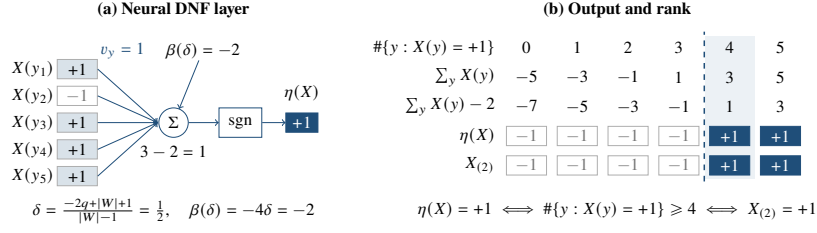

    \section{A smooth surrogate for rank filters}
    \label{SecRep}
    
	A smooth surrogate for rank filters may be obtained by combining Proposition \ref{main_prop} with Lemma \ref{lem1} and considering a smooth approximation of the Boolean variables in \eqref{rank_sum}. Let $\sigma$ be the $\tanh$ activation function and, for $\alpha > 0$, define
		$S_{\alpha}(t,a,f_{W}) \coloneqq \sum_{y \in W} \sigma\left(\alpha\left(f_{W}(y) - a(y) - t + 1/2\right)\right)$
	for a structuring element $a: W \to \mbR$, $t \in M$ and $f_{W} \in \mcF_{W}$. From now on, let $\beta(\delta) \coloneqq \delta(1 - |W|)$, which is obtained from \eqref{node} when $v_{y} = 1$ for all $y \in W$. Denote by $H(z) = \mathds{1}\{z \geq 0\}$  the Heaviside step function.
    
    By exchanging $\sum_y v_y X(y)$ in \eqref{node} by $S_{\alpha}(t,a,f_{W})$, applying $H$ instead of $\sign$ there and summing over $t$, we conclude that any rank filter may be approximated by a DNF layer applied to \textit{smoothed} Boolean variables with output in $\{0,1\}$ instead of $\{-1,+1\}$. If $\alpha$ is large enough, then this representation is exact in the integer scale as stated in the next corollary, for which a proof is presented in Section \ref{SecProof}.
	
	\begin{corollary}
		\label{main_cor}
		For every $a: W \to \mbZ, q = 1,\dots,|W|$ and $\alpha > \ln(2|W| - 1)$ it holds
			$\sup_{f \in \mcF} \sup_{x \in E} |\psi_{a,q}(f)(x) - \mcN_{a,\alpha,\delta}(f|_{W + x})| = 0$
		in which
		\begin{align}
			\label{NN}
			\mcN_{a,\alpha,\delta}(f|_{W + x}) = \sum_{t = 1}^{m} H\left(S_{\alpha}(t,a,f|_{W+x}) + \beta(\delta)\right)
		\end{align}
		and $\delta$ is computed from $q$ by \eqref{delta_formula}. In particular, since $S_{\alpha}(t,a,f|_{W+x})$ is non-increasing in $t$, $\mcN_{a,\alpha,\delta}(f|_{W + x}) = 0$ if $S_{\alpha}(0,a,f|_{W+x}) + \beta(\delta) < 0$ and
        \begin{equation}
            \label{exact_N_max}
            \mcN_{a,\alpha,\delta}(f|_{W + x}) = \max\{t \in M: S_{\alpha}(t,a,f|_{W+x}) + \beta(\delta) \geq 0\},
        \end{equation}
        otherwise.
	\end{corollary}

    Since $S_{\alpha}(t,a,f|_{W+x})$, as a function of $t \in \mbR$, is strictly decreasing in $t$ with limit $-|W|$ and $|W|$ when $t \to \pm \infty$, there exists a unique $\tst \coloneqq \tst(\alpha,\delta,a,f|_{W+x}) \in \mbR$ such that $S_{\alpha}(\tst,a,f|_{W+x}) + \beta(\delta) = 0$ and it holds $\mcN_{a,\alpha,\delta}(f|_{W + x}) = \left(\lfloor \tst \rfloor \vee 0\right) \wedge m$ by \eqref{exact_N_max}. We note that $\tst$ can be efficiently computed by any root finding algorithm for one-dimensional monotonic smooth functions in $\mbR$, for example Newton's method started at a suitable point.   
    
    As discussed in the introduction, the goal is to have a smooth surrogate for rank filters that gives an approximation to their subgradient with respect to the structuring element entries. Corollary \ref{main_cor}, and in particular \eqref{exact_N_max}, provides such a surrogate by (a) taking $\alpha$ smaller than the exact threshold, (b) considering $\tst(\alpha,\delta,a,f|_{W+x})$ without truncating to a greyscale level and (c) taking the structuring element $a:W \to \mbR$ as real-valued. In this case, the derivative of $\tst(\alpha,\delta,a,f|_{W+x})$ with respect to $a(y)$ follows by implicit differentiation.

    \begin{proposition}
		\label{prop_dev}
		Fix $\alpha > 0$, $\delta \in [-1,+1]$, $f_{W} \in \mcF_{W}$ and $a:W \to \mbR$. Denote
		\begin{align}
			\label{sigma_y}
			\sigma_{y}(t) \coloneqq \sigma\left(\alpha\left(f_{W}(y) - a(y) - t + 1/2\right)\right), \quad y \in W, t \in \mbR.
		\end{align}
		Then, it holds
		\begin{align}
            \label{derivp}
			\frac{\partial\tst}{\partial a(y)} = -\frac{1 - \sigma_{y}(\tst)^{2}}{\sum_{z \in W} \left(1 - \sigma_{z}(\tst)^{2}\right)}.
		\end{align}
		In particular, all derivatives are negative and $\sum_{y \in W} \partial \tst/\partial a(y) = -1$.
	\end{proposition}
	
	Extending \eqref{gen_filter} to real-valued $a$, the order statistic $[f|_{W+x}-a]_{(q)}$ is continuous and piecewise linear in $a$, with gradient $\partial [f|_{W+x}-a]_{(q)}/\partial a(y) = -\mathds{1}\{y = y_{q}\}$ whenever the $q$-th order statistic is attained at a unique $y_{q} \in W$. Under this uniqueness assumption, $\partial\tau^\star/\partial a(y)$ converges to this gradient as $\alpha \to \infty$. To see this, write $c_y := f_W(y) - a(y)$ and $u_y := \alpha(c_y - \tau^\star + 1/2)$, so that, with $\delta$ given by \eqref{delta_formula}, $\sum_{y} \sigma(u_y) = -\beta(\delta) = |W| + 1 - 2q$. Since $S_\alpha(t, a, f_W) + \beta(\delta)$ is positive at $t = c_{y_q} + 1/2 - \varepsilon$ and negative at $t = c_{y_q} + 1/2 + \varepsilon$ for sufficiently small \(\varepsilon>0\) and large enough $\alpha$, $\tau^\star - 1/2 \to c_{y_q}$. Hence $|u_y| \to \infty$ for $y \neq y_q$, so $1 - \sigma_y(\tau^\star)^2 \to 0$, whereas $\sigma(u_{y_q}) = |W| + 1 - 2q - \sum_{y \neq y_q} \sigma(u_y) \to 0$, so $1 - \sigma_{y_q}(\tau^\star)^2 \to 1$, and the claim follows from \eqref{derivp}. Thus, the surrogate is a smoothing of the true gradient and $\alpha$ has the role of \textit{distributing} the derivatives over the points in the window. Indeed, since $\sum_{y \in W} \partial \tst/\partial a(y) = -1$, the gradient with respect to $a$ is, up to a sign, a probability vector over $W$. In particular, $\alpha$ cannot rescale it, but only redistribute the probabilities, with large $\alpha$ concentrating the probability mass on the pixel attaining the $q$-th order statistic, while moderate $\alpha$ shares it across the window.

    The composition of rank filters, and in particular openings and closings, may be represented exactly by applying sequentially the layer $\mcN_{a,\alpha,\delta}$ with the respective structuring element and suitable value of $\delta$. Furthermore, the infimum and supremum of rank filters may be represented by adapting the layer in \eqref{NN} to compute the minimum ($\delta = +1$) or maximum ($\delta = -1$) of a collection $\{\psi_{a_{j},q}(f)(x): j = 1,\dots,n\}$ of $n$ rank filters applied to $f$ at point $x$. Formally, by the same proof as Corollary \ref{main_cor},
    \begin{align}
        \label{sup_rank}
    	\bigvee_{j=1}^{n} \psi_{a_{j},q}(f)(x) &= \sum_{t=1}^{m} H\Big(\sum_{j=1}^{n} \sigma\Big(\alpha\Big(\psi_{a_{j},q}(f)(x) - t + 1/2\Big)\Big) - (1 - n)\Big)
    \end{align}
    and the infimum is obtained by flipping the sign of $(1 - n)$ to positive, as long as $\alpha > \ln(2n - 1)$. We may define $\tst$ analogously as the root in $t$ of the expression inside $H$ in \eqref{sup_rank} and obtain a closed form derivative of $\tst$ with respect to each operator $\psi_{a_{j},q}(f)(x)$ as in Proposition \ref{prop_dev} which is necessary, by the chain rule, to compute the derivative of the supremum with respect to the entries of the structuring elements $a_{j}$.

    \section{Training exact rank filters with smoothed gradients}
	\label{SecTrain}

   	We propose to train rank filters by gradient descent with the \textit{exact} operator in the forward pass and the \textit{smooth} surrogate $\tst$ of Section~\ref{SecRep} in the backward pass, in what is known in the deep learning literature as a \textit{straight-through estimator} (STE) \cite{bengio2013estimating}.
	
	\subsection{Straight-through and derivatives}

    Following the STE framework, we consider the operator $\mcR_{a,\alpha,\delta}: \mcF_{W} \to \mbR$ given by
	\begin{align}
        \label{STE}
		\mcR_{a,\alpha,\delta}(f_{W}) = \tst(\alpha,\delta,a,f_{W}) + \text{sg}\left(\psi_{a,q}(f_{W}) - \tst(\alpha,\delta,a,f_{W})\right), \, f_{W} \in \mcF_W
	\end{align}
	in which $\text{sg}(\cdot)$ is the \textit{stop-gradient} function that is the identity function in the forward pass, and has gradient zero in the backward pass. In this formula, $q$ and $\delta$ satisfy relation \eqref{delta_formula}.
	
	In particular, it follows that, for a pair $(f,g) \in \mcF^{2}$ of input and output images and a finite frame $A \subset E$,
	\begin{align*}
		L(a,\delta;(f,g)) \coloneqq \frac{1}{|A|} \sum_{x \in A} (\mcR_{a,\alpha,\delta}(f|_{W+x}) - g(x))^{2} = \frac{1}{|A|} \sum_{x \in A} (\psi_{a,q}(f|_{W+x}) - g(x))^{2}
	\end{align*}
	is the MSE of the exact filter $\psi_{a,q}(f|_{W+x})$ while
	\begin{align}
		\label{deriv}
		\frac{\partial L}{\partial a(y)}(a,\delta;(f,g)) = \frac{2}{|A|} \sum_{x \in A} (\psi_{a,q}(f|_{W+x}) - g(x)) \frac{\partial \tst}{\partial a(y)}(\alpha,\delta,a,f|_{W+x})
	\end{align}
	in which the derivative of $\tst$ is given in Proposition \ref{prop_dev}. The equalities above are what we mean by exact on the forward pass and smooth on the backward pass, respectively. Since $S_{\alpha}$ depends on $f_{W}(y) - a(y)$, the derivative with respect to the input is $\partial \tst/\partial f_{W}(y) = -\partial \tst/\partial a(y)$, which propagates the gradient through compositions of rank filters.

    Since $-\partial \tst/\partial a$ is a probability vector over $W$ whose spread decreases with $\alpha$ (Section~\ref{SecRep}), a relatively \textit{low} value of $\alpha$ spreads the gradient and encourages the update of the whole structuring element. We emphasise that, whatever the choice of $\alpha$, the operator being trained is the exact $\psi_{a,q}$, and the surrogate is used only in the backward pass.

    The composition, infimum and supremum of rank filters may be trained in a similar fashion, by considering the respective composition, infimum or supremum of operators represented as in \eqref{STE}. Proceeding in this manner, the forward pass remains exact, but the backward pass will provide an approximation of the respective subgradient by backpropagation.

	To compute the derivative in \eqref{deriv}, the value of $\tst$ is needed in \eqref{derivp}. We compute it by Newton's method on $S_{\alpha}(t,a,f|_{W+x}) + \beta(\delta)$, or on the expression inside $H$ in \eqref{sup_rank}, started near the value of the forward pass, i.e., $t_{0} = \psi_{a,q}(f|_{W+x}) + 1/2$. The root is computed independently at each pixel, so it is parallel on a GPU. In the experiments, we use $4$ Newton steps for the rank filters and $8$ for the supremum \eqref{sup_rank}. Figure~\ref{fig_ste} illustrates the STE pipeline.
	
	 \begin{figure}[!t]
		\centering
		\resizebox{0.9\textwidth}{!}{\definecolor{rkdark}{RGB}{31,78,121}
\definecolor{rklight}{RGB}{217,226,233}
\definecolor{rkgrey}{RGB}{120,120,120}
\definecolor{rkback}{RGB}{204,85,0}
\definecolor{rkbacklight}{RGB}{251,232,216}
\begin{tikzpicture}[
    every node/.style = {inner sep = 1pt},
    fbox/.style = {draw = rkdark, fill = rklight, rounded corners = 2pt, align = center, inner sep = 3pt, minimum height = 0.8cm},
    bbox/.style = {draw = rkback, fill = rkbacklight, rounded corners = 2pt, align = center, inner sep = 3pt, minimum height = 0.8cm},
    fwd/.style = {-latex, rkdark, thick},
    bwd/.style = {-latex, rkback, thick, dashed}]

    \foreach \k in {1,0} \draw[rkgrey, fill = white] (0.15 + 0.1*\k,2.0 + 0.1*\k) rectangle ++(0.8,0.8);
    \node at (0.55,2.4) {$f$};
    \foreach \i in {0,1,2} \foreach \j in {0,1,2} \draw[rkgrey, fill = rkbacklight] (0.25 + 0.2*\i,0.15 + 0.2*\j) rectangle ++(0.2,0.2);
    \node[below] at (0.55,0.15) {$a$};

    \node[text = rkdark, font = \small\bfseries, anchor = south west] at (0,3.0) {forward pass: exact};
    \node[fbox] (rank) at (4.1,2.4) {exact rank filter\\$\psi_{a,q}(f)(x) = [f|_{W+x} - a]_{(q)}$};
    \foreach \k in {1,0} \draw[rkgrey, fill = rklight] (8.2 + 0.1*\k,2.0 + 0.1*\k) rectangle ++(0.8,0.8);
    \node at (8.65,3.15) {$\psi_{a,q}(f)$};
    \node[draw = rkdark, circle, minimum size = 0.7cm] (loss) at (12.6,2.4) {$L$};
    \node (g) at (12.6,3.3) {target $g$};
    \draw[-latex, rkgrey] (g) -- (loss);
    \draw[fwd] (1.15,2.4) -- (rank);
    \draw[fwd] (0.95,0.75) -- (rank.south west);
    \draw[fwd] (rank) -- (8.2,2.4);
    \draw[fwd] (9.2,2.4) -- (loss);

    \node[text = rkback, font = \small\bfseries, anchor = north west] at (0,-0.75) {backward pass: smooth};
    \node[bbox] (tau) at (10.1,0.45) {root $\tau^{\star}$ of $S_{\alpha}(t,a,f|_{W+x}) + \beta(\delta)$\\Newton from $t_{0} = \psi_{a,q}(f)(x) + 1/2$};
    \node[bbox] (grad) at (4.4,0.45) {$\dfrac{\partial L}{\partial a(y)} = \dfrac{\partial L}{\partial \psi}\,\dfrac{\partial \tau^{\star}}{\partial a(y)}$\\[4pt]$\dfrac{\partial \tau^{\star}}{\partial a(y)} = -\dfrac{1 - \sigma_{y}(\tau^{\star})^{2}}{\sum_{z \in W} \left(1 - \sigma_{z}(\tau^{\star})^{2}\right)}$};
    \draw[bwd] (loss) |- node[pos = 0.35, right, font = \small, text = rkback] {$\dfrac{\partial L}{\partial \psi}$} (tau.east);
    \draw[bwd] (tau) -- (grad);
    \draw[bwd] (grad) -- (0.95,0.45);
    \draw[-latex, rkgrey, dotted, thick] (8.65,1.95) -- node[right, font = \scriptsize, text = rkgrey, align = left] {straight-through: value $\psi_{a,q}$,\\gradient through $\tau^{\star}$} (8.65,0.92);
\end{tikzpicture}}
		\caption{\footnotesize Training an exact rank filter. The forward pass evaluates the exact filter $\psi_{a,q}$, so the loss is that of the exact operator. The backward pass replaces its sparse subgradient by the derivative of the root $\tau^{\star}$ of the smooth DNF threshold equation.}
		\label{fig_ste}
	\end{figure}
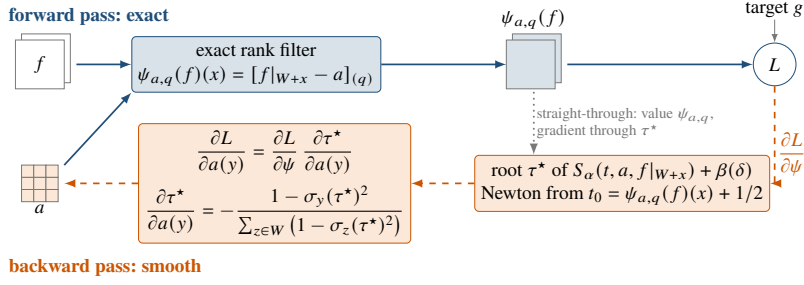

    \section{Experiments}
	\label{SecExperiments}

    In order to illustrate the method, we consider two experiments. The first one follows the denoising experiment of \cite{bottenmuller2025improving} on MNIST \cite{lecun1998gradient} and F-MNIST \cite{xiao2017fashion}, in which we compare networks of rank filters trained with our method and with the exact subgradient to their HSMorph layer, the harmonised version of the SMorph layer of \cite{hermary2022learning}, trained with and without their constraint (HSMorph-C and HSMorph). The second is the segmentation of retinal vessels on the CHASE\_DB1 dataset \cite{fraz2012ensemble}. After a black top-hat transformation of the green channel of the images, we learn a supremum of erosions with our method and with the exact subgradient. The experiments were performed on an NVIDIA H200 141GB GPU and code in JAX will be provided. 
    
    The structuring elements are initialised from independent $\mathcal{N}(0,0.3^{2})$ entries. Our method uses $\alpha = 10$, well below the exactness threshold of the smooth representation (in Experiment 1, $\alpha>510\ln17\approx1445$), to spread out the gradients across the window. In preliminary experiments we observed that better results are obtained for $\alpha \in [5,20]$. The exact subgradient is obtained by sub-differentiating the exact operators. The rank filters are applied without the truncation to $[0,m]$ in \eqref{gen_filter}, in training and in evaluation. In all cases, we minimise the MSE with Adam with default parameters, learning rate $1/255$ halved every $3{,}750$ steps, for at most $100$ epochs, stopping when the training MSE has not improved by $1\%$ for $10$ epochs. We return the parameters of the epoch with the least training MSE on the whole dataset. Each case is run with $10$ seeds, which set the initial parameters, the batch order and, in Experiment 1, the noise.

    \subsection{Experiment 1: image denoising}
    
    The $60{,}000$ training and $10{,}000$ test images of each dataset are normalised to $[0,1]$, with $m = 255$, and then mapped to $[0.5,1]$. After this mapping, the grey levels are $1/510$ apart and the threshold for exactness becomes $\alpha > 510\,\ln(2|W| - 1)$. For $p \in \{1,5,10,25,50\}$, $p/2\%$ of the pixels of every image are set to $0$ and $p/2\%$ to $1$, at positions chosen uniformly at random. Networks are trained with batch size $32$ to recover the clean image. For our method and the exact subgradient, we consider the following compositions of $3 \times 3$ rank filters with independent structuring elements: ero-dila-dila-ero, dila-ero-ero-dila, a median, two medians, and a median followed by each of the first two. For networks that differ only in the order of the erosions and dilations, we report, for each seed, the order with the least training MSE. 
    
    Following \cite{bottenmuller2025improving}, HSMorph networks have four $3 \times 3$ layers followed by a linear layer, with structuring elements initialised as above, the linear scale as $1$ and bias as $0$, and with the input standardised by the mean and standard deviation of the noisy training images. The HSMorph-C networks have the same architecture, but with a penalty term of weight $0.01$ that favours an alternating sequential filter: layers 1-2 and 3-4 are pushed towards opposite operator types, and layers 2-3 towards the same type (see \cite{bottenmuller2025improving} for more details). We report the test PSNR with peak $1$, averaged over the test images.

    \begin{table}[!t]
    	\centering
    	\caption{\footnotesize Test PSNR (dB) of the denoising networks, mean $\pm$ standard deviation over $10$ seeds. For each dataset and noise level, the best mean is in bold. ASF: four alternating erosion and dilation layers.}
    	\label{tab_denoising}
    	\scriptsize
    	\setlength{\tabcolsep}{3pt}
    	\begin{tabular}{llccccc}
    		\hline
    		Network & Training & $1\%$ & $5\%$ & $10\%$ & $25\%$ & $50\%$ \\
    		\hline
    		\multicolumn{7}{c}{MNIST} \\
    		\hline
    		\multirow{2}{*}{ASF} & Ours & $\mathbf{33.03 \pm 1.41}$ & $\mathbf{29.56 \pm 0.60}$ & $26.88 \pm 0.73$ & $22.79 \pm 0.38$ & $18.78 \pm 0.09$ \\
    		& Exact subgradient & $31.47 \pm 1.56$ & $27.58 \pm 0.77$ & $25.98 \pm 0.77$ & $22.30 \pm 0.46$ & $18.50 \pm 0.30$ \\
    		\multirow{2}{*}{Median} & Ours & $29.70 \pm 1.42$ & $28.16 \pm 1.16$ & $26.87 \pm 1.01$ & $22.20 \pm 0.60$ & $15.20 \pm 0.12$ \\
    		& Exact subgradient & $28.87 \pm 1.99$ & $27.09 \pm 1.73$ & $25.39 \pm 1.47$ & $21.35 \pm 0.80$ & $14.93 \pm 0.24$ \\
    		\multirow{2}{*}{Median $\times$ 2} & Ours & $28.73 \pm 1.05$ & $27.67 \pm 0.78$ & $26.73 \pm 0.52$ & $23.95 \pm 0.30$ & $18.80 \pm 0.19$ \\
    		& Exact subgradient & $28.66 \pm 1.70$ & $26.71 \pm 1.23$ & $25.76 \pm 1.03$ & $23.35 \pm 0.39$ & $18.56 \pm 0.28$ \\
    		\multirow{2}{*}{Median + ASF} & Ours & $31.03 \pm 2.12$ & $28.92 \pm 1.00$ & $\mathbf{27.04 \pm 0.57}$ & $\mathbf{24.02 \pm 0.32}$ & $\mathbf{19.93 \pm 0.13}$ \\
    		& Exact subgradient & $30.34 \pm 1.59$ & $27.47 \pm 1.31$ & $26.43 \pm 0.77$ & $23.65 \pm 0.36$ & $19.77 \pm 0.19$ \\
    		\multicolumn{2}{l}{HSMorph-C} & $31.64 \pm 1.58$ & $28.75 \pm 0.50$ & $26.59 \pm 0.46$ & $21.85 \pm 0.77$ & $18.58 \pm 0.56$ \\
    		\multicolumn{2}{l}{HSMorph} & $32.81 \pm 0.27$ & $28.86 \pm 0.12$ & $26.16 \pm 0.55$ & $20.29 \pm 0.08$ & $18.16 \pm 0.01$ \\
    		\multicolumn{2}{l}{Noisy input} & $25.56$ & $18.50$ & $15.59$ & $11.58$ & $8.57$ \\
    		\hline
    		\multicolumn{7}{c}{F-MNIST} \\
    		\hline
    		\multirow{2}{*}{ASF} & Ours & $29.69 \pm 0.44$ & $27.22 \pm 0.37$ & $25.52 \pm 0.19$ & $22.94 \pm 0.13$ & $19.65 \pm 0.08$ \\
    		& Exact subgradient & $29.57 \pm 0.64$ & $26.89 \pm 0.75$ & $24.91 \pm 0.51$ & $22.66 \pm 0.25$ & $19.27 \pm 0.34$ \\
    		\multirow{2}{*}{Median} & Ours & $30.37 \pm 1.91$ & $28.27 \pm 1.66$ & $26.15 \pm 1.14$ & $21.53 \pm 0.53$ & $15.36 \pm 0.10$ \\
    		& Exact subgradient & $29.06 \pm 1.87$ & $27.44 \pm 2.09$ & $25.02 \pm 1.18$ & $21.02 \pm 0.75$ & $15.20 \pm 0.21$ \\
    		\multirow{2}{*}{Median $\times$ 2} & Ours & $\mathbf{30.42 \pm 1.41}$ & $\mathbf{28.85 \pm 1.11}$ & $\mathbf{26.68 \pm 0.66}$ & $23.68 \pm 0.23$ & $19.38 \pm 0.13$ \\
    		& Exact subgradient & $29.60 \pm 1.71$ & $27.22 \pm 1.64$ & $25.57 \pm 0.73$ & $23.59 \pm 0.27$ & $19.21 \pm 0.28$ \\
    		\multirow{2}{*}{Median + ASF} & Ours & $30.34 \pm 1.11$ & $28.20 \pm 1.23$ & $26.65 \pm 0.80$ & $\mathbf{23.75 \pm 0.17}$ & $\mathbf{20.93 \pm 0.09}$ \\
    		& Exact subgradient & $29.38 \pm 0.60$ & $26.96 \pm 1.03$ & $25.49 \pm 0.88$ & $23.57 \pm 0.18$ & $20.78 \pm 0.21$ \\
    		\multicolumn{2}{l}{HSMorph-C} & $28.35 \pm 1.05$ & $25.93 \pm 0.55$ & $24.92 \pm 0.51$ & $22.85 \pm 0.18$ & $19.13 \pm 0.71$ \\
    		\multicolumn{2}{l}{HSMorph} & $28.94 \pm 0.47$ & $26.56 \pm 0.30$ & $25.06 \pm 0.33$ & $21.82 \pm 0.85$ & $18.59 \pm 0.05$ \\
    		\multicolumn{2}{l}{Noisy input} & $25.23$ & $18.17$ & $15.26$ & $11.25$ & $8.24$ \\
    		\hline
    	\end{tabular}
    \end{table}

    Table~\ref{tab_denoising} shows that, for every network and noise level, training with our method gives a higher mean test PSNR than training with the exact subgradient, with gains of up to $2.0$~dB (ASF on MNIST at 5\% noise) and, in most cases, a smaller standard deviation over the seeds. Since both methods are evaluated by the same exact operators, these differences are due only to the smooth backward pass, whose gain is largest at moderate noise ($1.1$~dB on average at 5\% and 10\%). The best of our four networks also outperforms the smooth HSMorph networks in all ten columns, by a small margin on MNIST at 1\% noise. The margin over HSMorph-C reaches $2.2$~dB on MNIST (25\% noise) and $2.9$~dB on F-MNIST (5\% noise). Our best networks also improve on the best results reported in \cite{bottenmuller2025improving} in nine of the ten settings, the exception being MNIST at 50\% noise ($19.93$ against $20.12$~dB). Hence, our method outperforms smooth morphological networks while the trained network remains a composition of exact rank filters, and the operator that is learned is exactly the one that is evaluated.

    \subsection{Experiment 2: retinal vessels}
    
    The CHASE\_DB1 dataset has $28$ images of $999 \times 960$ pixels labelled by two observers, and we use the first $20$ to train and the last $8$ to test, with the vessels of the first observer as target. The input is the black top-hat of the green channel by a disk of radius $12$, divided by its median over the vessel pixels of the training images. We call foreground the set where the red channel exceeds $20/255$, eroded by a disk of radius $4$. The network is the supremum of $256$ erosions by $15 \times 15$ structuring elements, trained by our method and by the exact subgradient. We consider a batch size $1$ on the tiles of $64 \times 64$ pixels of a grid covering each training image, keeping those with more than half of their pixels in the foreground ($3{,}203$ tiles), with the binary vessel image as target. Metrics are computed on the foreground pixels of the test images: the AUC of the output, and the accuracy, sensitivity, specificity and F1 at the threshold maximising the training F1.

    \begin{table}[!t]
    	\centering
    	\caption{\footnotesize Test metrics of the supremum of $256$ erosions on CHASE\_DB1, computed on the foreground, mean $\pm$ standard deviation over $10$ seeds, and of the top-hat input thresholded in the same way. The best mean of each metric is in bold.}
    	\label{tab_vessels}
    	\scriptsize
    	\setlength{\tabcolsep}{3pt}
    	\begin{tabular}{lccccc}
    		\hline
    		Training & AUC & F1 & Accuracy & Sensitivity & Specificity \\
    		\hline
    		Ours & $\mathbf{0.9279 \pm 0.0002}$ & $\mathbf{0.6415 \pm 0.0004}$ & $\mathbf{0.9303 \pm 0.0003}$ & $\mathbf{0.6765 \pm 0.0025}$ & $\mathbf{0.9561 \pm 0.0005}$ \\
    		Exact subgradient & $0.9096 \pm 0.0015$ & $0.6129 \pm 0.0024$ & $0.9248 \pm 0.0008$ & $0.6459 \pm 0.0086$ & $0.9531 \pm 0.0016$ \\
    		Top-hat input, no learning & $0.9042$ & $0.5834$ & $0.9166$ & $0.6336$ & $0.9453$ \\
    		\hline
    	\end{tabular}
    \end{table}
    
    \begin{figure}[!t]
    	\centering
    	\includegraphics[width=\textwidth]{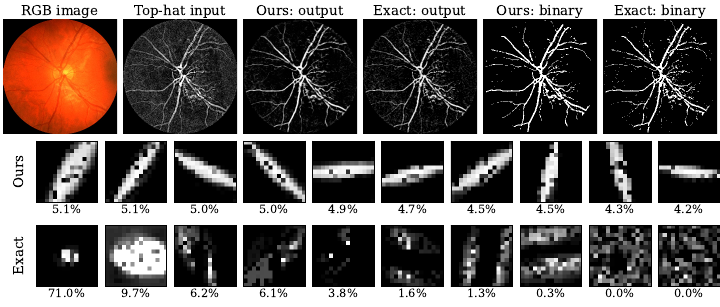}
    	\caption{\footnotesize Retinal vessels, seed $0$. Top: a test image of CHASE\_DB1, its black top-hat input, and the output of the supremum of $256$ erosions trained by our method and by the exact subgradient, in greyscale and thresholded at the value that maximises the F1 on the training images, on the foreground. Bottom: structuring elements of the $10$ erosions that most often attain the supremum on the vessel pixels of the training images detected by the network (labelled as vessel and above the threshold), with the share of these pixels below each. Each element $a$ is shown as $a - \max a$, from white ($0$) to black ($\leq -1$).}
    	\label{fig_vessels}
    \end{figure}
    
    Table~\ref{tab_vessels} shows that our method improves on the exact subgradient in all five metrics, whereas the exact subgradient improves little on the top-hat input. Since both networks are evaluated by the same exact operators, the difference is due only to the backward pass. In Figure~\ref{fig_vessels}, the erosions most active in our method are line segments of different widths and orientations, the shapes expected to fit into vessels~\cite{zana2001segmentation}, and the detected vessel pixels are shared among $182$ erosions, the top-10 accounting for $47\%$ of them. With the exact subgradient, only $21$ erosions are active and two of them account for $81\%$. This reflects the gradient sparsity of exact operators \cite{blusseau2024training,dimitrova2026optimizing}: erosions that rarely attain the supremum are rarely updated, and most structuring elements remain close to their initialisation, with $75\%$ of their entries never updated (fewer than $0.1\%$ with our method) and a median relative change $\|a - a_{0}\|/\|a_{0}\|$ of $0.05$ (against $1.27$).

    \section{Final remarks}
\label{SecFinal}

In this paper, we proposed a principled way to train exact rank filters with gradient descent via a straight-through estimator, in which the backward pass relies on the threshold decomposition combined with the DNF neural layer. The results show that our method is superior to training with the exact subgradient and is competitive with, if not better than, smooth approximations of erosions and dilations, underscoring that it is possible to efficiently train morphological operators via gradient descent without losing the exactness. In particular, our method applies to any rank filter, while, in general, the smooth layers proposed in the literature approximate only erosions and dilations. Since compositions, infima and suprema of rank filters are trained in the same way, the procedure extends to grey-level images the framework of discrete morphological neural networks \cite{marcondes2024discrete,marcondes2025unrestricted}, in which the architecture is designed from prior knowledge and the trained operator remains exact and interpretable.

We leave many topics to be explored in future research. First, it would be interesting to learn the rank $\delta$ so the architecture is not fixed beforehand, which is closer to the framework of \cite{bottenmuller2025improving,hermary2022learning}. Second, the experiments illustrate that our method is less sensitive than the exact subgradient to the initial parameters, which are selected independently and with high variance, a fact that needs further investigation. Finally, the proposed framework may be extended to non-increasing operators by representing them as suprema of sup-generating operators, following the Banon-Barrera decomposition \cite{banon1993decomposition}, and other surrogates, such as smooth morphological layers, could be used in the backward pass.

    \section{Proofs}
\label{SecProof}

\begin{proof}[Proof of Lemma \ref{lem1}]
	Let $X^{+} = \{y \in W: X(y) = +1\}$, so that $\sum_{y \in W} X(y) = 2|X^{+}| - |W|$ and $X_{(q)} = +1$ iff $|X^{+}| \geq |W| - q + 1$. By \eqref{delta_formula}, $\beta(\delta) = \delta\gamma(1 - |W|) = \gamma(2q - |W| - 1)$, hence
	\begin{align*}
		\gamma \sum_{y \in W} X(y) + \beta(\delta) = \gamma\left(2(|X^{+}| + q - |W|) - 1\right),
	\end{align*}
	which is at least $\gamma > 0$ if $X_{(q)} = +1$ and at most $-\gamma < 0$ if $X_{(q)} = -1$, so $\eta(X) = X_{(q)}$. The particular cases follow from $q = |W|$, $q = 1$ and $q = (|W|+1)/2$ in \eqref{delta_formula}. \hfill$\blacksquare$
\end{proof}

\begin{proof}[Proof of Corollary \ref{main_cor}]
	Since $f_{W}(y)-a(y)-t$ is an integer, $|\alpha(f_{W}(y)-a(y)-t+1/2)| \geq \alpha/2$, and $\alpha > \ln(2|W|-1)$ gives $1-\sigma(\alpha/2) < 1/|W|$. Hence each term of $S_{\alpha}(t,a,f_{W})$ differs from $\sign\{f_{W}(y)-a(y)-t\}$ by less than $1/|W|$, and $S_{\alpha}(t,a,f_{W})+\beta(\delta)$ differs from $\sum_{y \in W}\sign\{f_{W}(y)-a(y)-t\}+\beta(\delta)$ by less than one. By the proof of Lemma~\ref{lem1} with $\gamma=1$, the latter is at least $1$ if $[\left(\mathds{1}\{f_{W}(y)-a(y)\geq t\}: y\in W\right)]_{(q)}=1$ and at most $-1$ otherwise, so both have the same sign, and the result follows from \eqref{rank_sum}. \hfill$\blacksquare$
\end{proof}

\begin{proof}[Proof of Proposition \ref{prop_dev}]
	Let $F(t,a) \coloneqq S_{\alpha}(t,a,f_{W}) + \beta(\delta)$, so that $F(\tst,a) = 0$. Since $\sigma' = 1-\sigma^{2}$, we have $\partial F/\partial t = -\alpha\sum_{z \in W}(1-\sigma_{z}(t)^{2}) < 0$ and $\partial F/\partial a(y) = -\alpha(1-\sigma_{y}(t)^{2})$, and the implicit function theorem gives $\partial \tst/\partial a(y) = -(\partial F/\partial a(y))/(\partial F/\partial t)$ at $t=\tst$, which is \eqref{derivp}. \hfill$\blacksquare$
\end{proof}

\section*{Acknowledgements}

This work was supported by computational resources provided by the Australian Government through the National Computational Infrastructure (NCI) under the ANU Startup Scheme. %

	\bibliographystyle{splncs04}
	\bibliography{Ref}

\end{document}